\documentclass[wcp]{jmlr}

\renewcommand{\cite}{\citep}

\jmlrvolume{vol}
\jmlryear{2012}
\jmlrworkshop{On-line Trading of Exploration and Exploitation 2}

\title[PAC-Bayes-Bernstein Inequality for Martingales and its Application]{PAC-Bayes-Bernstein Inequality for Martingales\\and its Application to Multiarmed Bandits}

\author{\Name{Yevgeny Seldin} \Email{seldin@tuebingen.mpg.de}\\
\addr Max Planck Institute for Intelligent Systems, T\"{u}bingen, Germany
\AND
\Name{Nicol{\`o} Cesa-Bianchi} \Email{nicolo.cesa-bianchi@unimi.it}\\
\addr Dipartimento di Scienze dell''Informazione, Universit{\`a} degli Studi di Milano, Italy
\AND
\Name{Peter Auer} \Email{auer@unileoben.ac.at}\\
\addr Chair for Information Technology, University of Leoben, Austria
\AND
\Name{Fran\c{c}ois Laviolette} \Email{francois.laviolette@ift.ulaval.ca}\\
\addr Universit\'{e} Laval, Qu\'{e}bec, Canada
\AND
\Name{John Shawe-Taylor} \Email{jst@cs.ucl.ac.uk}\\
\addr University College London, UK
}

\editor{Editor's name}

\begin{document} 

\maketitle

\begin{abstract}
We develop a new tool for data-dependent analysis of the exploration-exploitation trade-off in learning under limited feedback. Our tool is based on two main ingredients. The first ingredient is a new concentration inequality that makes it possible to control the concentration of weighted averages of multiple (possibly uncountably many) simultaneously evolving and interdependent martingales.\footnote{See also our follow-up work on PAC-Bayesian inequalities for martingales \cite{SLCB+11}} The second ingredient is an application of this inequality to the exploration-exploitation trade-off via importance weighted sampling. We apply the new tool to the stochastic multiarmed bandit problem, however, the main importance of this paper is the development and understanding of the new tool rather than improvement of existing algorithms for stochastic multiarmed bandits. In the follow-up work we demonstrate that the new tool can improve over state-of-the-art in structurally richer problems, such as stochastic multiarmed bandits with side information \cite{SAL+11}.
\end{abstract}

\section{Introduction}

Learning under limited feedback and the exploration-exploitation trade-off are the fundamental questions in fields like reinforcement and active learning. The existing theoretical analysis of the exploration-exploitation trade-off in problems that go beyond multiarmed bandits is mainly focused on the worst-case scenarios \cite{SLL09, JOA10, BLL+11, BDL09}. But the worst-case analysis is overly pessimistic if the environment is not adversarial and cannot exploit the opportunities provided by benign conditions. We present a new analysis framework that lays the foundation for data-dependent analysis of the exploration-exploitation trade-off.

Our framework is based on PAC-Bayesian analysis. The PAC-Bayesian analysis was introduced over a decade ago \cite{STW97,ST+98,McA98,See02} and has since made a significant contribution to the analysis and development of supervised learning methods. PAC-Bayesian bounds provide an explicit and often intuitive and easy-to-optimize trade-off between model complexity and empirical data fit, where the complexity can be nailed down to the resolution of individual hypotheses via the definition of the prior. The PAC-Bayesian analysis was applied to derive generalization bounds and new algorithms for linear classifiers and maximum margin methods \cite{LST02, McA03b, GLLM09}, structured prediction \cite{McA07}, and clustering-based classification models \cite{ST10}, to name just a few. However, the application of PAC-Bayesian analysis beyond the supervised learning domain remained surprisingly limited. In fact, the only additional domain known to us is density estimation \cite{ST10, HST10}.

Application of PAC-Bayesian analysis to non-i.i.d. data was partially addressed only recently by \citet{RSS10} and \citet{LLST10}. The solution of Ralaivola et al. is based on breaking the sample into independent (or almost independent) subsets (which also reduces the effective sample size to the number of independent subsets). Such an approach is inapplicable in reinforcement learning due to strong dependence of the learning process on all of its history. Lever et al. treated dependent samples in the context of analysis of U-statistics. They employed Hoeffding's canonical decomposition of U-statistics into forward martingales and applied PAC-Bayesian analysis directly to these martingales. The approach presented here is both tighter and more general.

We present a generalization of PAC-Bayesian analysis to martingales. Our generalization makes it possible to consider model order selection simultaneously with the exploration-exploitation trade-off. Some potential advantages of applying PAC-Bayesian analysis in reinforcement learning were recently pointed out by several researchers, including \citet{TP10} and \citet{FP10}. Tishby and Polani suggested to use the mutual information between states and actions in a policy as a natural regularizer in reinforcement learning. They showed that regularization by mutual information can be incorporated into Bellman equations and thereby computed efficiently. Tishby and Polani conjectured that PAC-Bayesian analysis can be applied to justify such a regularization and provide generalization guarantees for it.

Fard and Pineau derived a PAC-Bayesian analysis of batch reinforcement learning. However, batch reinforcement learning does not involve the exploration-exploitation trade-off.

One of the reasons for the difficulty of applying PAC-Bayesian analysis to address the exploration-exploitation trade-off is limited feedback (the fact that we only observe the reward for the action taken, but not for all other actions). In supervised learning (and also in density estimation) the empirical error of each hypothesis in a hypotheses class can be evaluated on all the samples and, therefore, the size of the sample available for evaluation of all the hypotheses is the same (and usually relatively large). In the situation of limited feedback the samples from one action cannot be used to evaluate another action and the sample size of ``bad'' actions has to increase sublinearly in the number of game rounds. In a precursory report \cite{SLST+11} we overcame this difficulty by applying PAC-Bayesian analysis to importance weighted sampling \cite{SB98}. Importance weighted sampling is commonly used in the analysis of non-stochastic bandits \cite{ACB+02}, but has not previously been applied to the analysis of stochastic bandits.

The usage of importance weighted sampling introduces two new difficulties. One is sequential dependence of the samples: the rewards observed in the past influence distribution over actions played in the future and through this distribution the variance of the subsequent weighted sample variables. The second problem introduced by weighted sampling is the growing variance of the weighted sample variables. In \citet{SLST+11} we handled this dependence by combining PAC-Bayesian analysis with Hoeffding-Azuma-type inequalities for martingales. The bounds achieved by such a combination provide $O(\frac{1}{\varepsilon_t \sqrt{t}})$ convergence rate, where $t$ is the time step and  $\varepsilon_t$ is the minimal probability of sampling any action at time step $t$. The combination with Bernstein-type inequality for martingales presented here achieves $O(\frac{1}{\sqrt{\varepsilon_t t}})$ convergence rate. This improvement makes it possible to tighten the regret bounds from $O(K^{1/2} t^{3/4})$ to $O(K^{1/3} t^{2/3})$, where $K$ is the number of arms. In Section \ref{sec:exp} we suggest possible ways to tighten the analysis further to get $O(\sqrt{Kt})$ regret bounds. These further improvements will be studied in detail in future work.

We repeat that our main goal is not improvement of existing bounds for stochastic multiarmed bandits, which are already tight up to $\sqrt{\ln(K)}$ factors \cite{AB09,AO10}, but rather development of a new powerful tool for reinforcement learning and for other domains with richer structure. The multiarmed bandits serve us as a testbed for the development of this new tool. One example of a problem with a richer structure are multiarmed bandits with side information (a.k.a. contextual bandits). \citet{BLL+11} suggested $O\left(\sqrt{Kt \ln(N/\delta)}\right)$ and $O\left(\sqrt{t(d\ln t - \ln \delta)}\right)$ regret bounds for learning with expert advice in multiarmed bandits with side information, where $N$ is the number of experts (in case it is finite) and $d$ is the VC-dimension of the set of experts (in case it is infinite). In the follow-up paper \citet{SAL+11} we show that PAC-Bayesian analysis makes it possible to replace $\ln(N)$ and $d$ factors with $KL(\rho\|\mu)$, where $KL$ is the KL-divergence, $\rho(h)$ is a distribution over the experts played by the algorithm, and $\mu(h)$ is a prior distribution over the experts. Such an approach is much more flexible, since it allows individual treatment of different experts (or policies) via the definition of the prior $\mu$.

The paper is organized as follows: Section \ref{sec:main} surveys the main results of the paper, Section \ref{sec:exp} suggests possible ways to tighten the analysis further, and Section \ref{sec:dis} discusses the results. Proofs are provided in the appendix.

\section{Main Results}
\label{sec:main}

We start with a general concentration result for martingales based on combination of PAC-Bayesian analysis with a Bernstein-type inequality for martingales. Then, we apply this result to derive an instantaneous (per-round) bound on the distance between expected and empirical regret for the multiarmed bandit problem. This result is in turn applied to derive an instantaneous regret bound for the multiarmed bandits.

\subsection{PAC-Bayes-Bernstein Inequality for Martingales}

In order to present our concentration result for martingales we need a few definitions. Let ${\cal H}$ be an index (or a hypothesis) space, possibly uncountably infinite. Let $\{X_1(h), X_2(h), \dots: h \in {\cal H}\}$ be martingale difference sequences, meaning that $\mathbb E [X_t(h)|{\cal T}_{t-1}] = 0$, where ${\cal T}_t = \{X_\tau(h):1\leq \tau \leq t \text{ and } h \in {\cal H}\}$ is a set of martingale differences observed up to time $t$ (the history). ($\{X_t(h)\}_{h \in {\cal H}}$ do not have to be independent, we only need the requirement on the conditional expectation to be satisfied.) Let $M_t(h) = \sum_{\tau = 1}^t X_\tau(h)$ be martingales corresponding to the martingale difference sequences and let $V_t(h) = \sum_{\tau=1}^t \mathbb E [X_\tau(h)^2|{\cal T}_{\tau-1}]$ be cumulative variances of the martingales. For a distribution $\rho$ over ${\cal H}$ define weighted averages of the martingales and their cumulative variances with respect to $\rho$ as $M_t(\rho) = \mathbb E_{\rho(h)} [M_t(h)]$ and $V_t(\rho) = \mathbb E_{\rho(h)} [V_t(h)]$.

\begin{theorem}[PAC-Bayes-Bernstein Inequality]
\label{thm:PAC-Bayes-Bernstein}
Let $\{C_1,C_2,\dots\}$ be an increasing sequence set in advance, such that $|X_t(h)| \leq C_t$ for all $h$ with probability 1. Let $\{\mu_1, \mu_2,\dots\}$ be a sequence of ``reference'' $($``prior''$)$ distributions over ${\cal H}$, such that $\mu_t$ is independent of ${\cal T}_t$ $($but can depend on $t$$)$. Let $\{\lambda_1, \lambda_2, \dots\}$ be a sequence of positive numbers set in advance that satisfy$:$
\begin{equation}
\label{eq:technical}
\lambda_t \leq \frac{1}{C_t}.
\end{equation}
Then for all possible distributions $\rho_t$ over ${\cal H}$ given $t$ and for all $t$ simultaneously with probability greater than $1-\delta$$:$
\begin{equation}
|M_t(\rho_t)| \leq \frac{KL(\rho_t\|\mu_t) + 2 \ln(t+1) + \ln \frac{2}{\delta}}{\lambda_t} + (e-2) \lambda_t V_t(\rho_t).
\label{eq:PAC-Bayes-Bernstein}
\end{equation}
\end{theorem}

Bound \eqref{eq:PAC-Bayes-Bernstein} is minimized by $\lambda_t = \sqrt{\frac{KL(\rho_t\|\mu_t) + 2 \ln(t+1) + \ln \frac{2}{\delta}}{(e-2)V_t(\rho_t)}}$. For this value of $\lambda_t$ we would get 
\begin{equation}
\label{eq:dream}
|M_t(\rho_t)| \leq 2 \sqrt{(e-2)V_t(\rho_t) \left(KL(\rho_t\|\mu_t) + 2 \ln(t+1) + \ln \frac{2}{\delta} \right)},
\end{equation}
however, $\lambda_t$ has to be set in advance and cannot depend on the sample. Therefore, we have to make our best guess of what the values of $KL(\rho_t\|\mu_t)$ and $V_t(\rho_t)$ are going to be, which is actually possible in the case that we study below. In the follow-up paper we show that by taking an exponentially spaced grid of $\lambda_t$-s and a union bound over this grid it is possible to derive a bound, which is almost as good as \eqref{eq:dream} \cite{SLCB+11}, but this extension is not required in the current work.

\subsection{Application to the Multiarmed Bandit Problem}

In order to apply our result to the multiarmed bandit problem we need some more definitions. Let ${\cal A}$ be a set of actions (arms) of size $|{\cal A}| = K$ and let $a \in {\cal A}$ denote the actions. Denote by $R(a)$ the expected reward of action $a$. Let $\pi_t$ be a distribution over ${\cal A}$ that is played at round $t$ of the game (a policy). Let $\{A_1,A_2,\dots\}$ be the sequence of actions played independently at random according to $\{\pi_1,\pi_2,\dots\}$ respectively. Let $\{R_1,R_2,\dots\}$ be the sequence of observed rewards. Denote by ${\cal T}_t = \left \{\{\pi_1,\dots,\pi_t\}, \{A_1,\dots,A_t\},\{R_1,\dots,R_t\}\right\}$ the set of played policies, taken actions, and observed rewards up to round $t$.

For $t\geq 1$ and $a \in \{1,\dots,K\}$ define a set of random variables $R_t^a$ (the importance weighted samples):
\[
R_t^a = \left \{ \begin{array}{cl}\frac{1}{\pi_t(a)}R_t,&\mbox{if}~A_t=a\\0,&\mbox{otherwise.}\end{array} \right .
\]
Define: 
\[
\hat R_t(a) = \frac{1}{t} \sum_{\tau=1}^t R_\tau^a.
\]
Observe that $\mathbb E[R^a_t|{\cal T}_{t-1}] = R(a)$ and $\mathbb E \hat R_t(a) = R(a)$.

Let $a^*$ be the ``best'' action (the action with the highest expected reward, if there are multiple ``best'' actions pick any of them). Define the expected and empirical per-round regrets as:
\begin{align*}
\Delta(a) &= R(a^*) - R(a),\\
\hat \Delta_t(a) &= \hat R_t(a^*) - \hat R_t(a).
\end{align*}

Observe that $t (\hat \Delta_t(a) - \Delta(a) )$ form a martingale. Let
\[
V_t(a) = \sum_{\tau=1}^t \mathbb E[([R_\tau^{a^*} - R_\tau^a] - [R(a^*) - R(a)])^2 | {\cal T}_{\tau-1}]
\]
be the cumulative variance of this martingale. 

Let $\{\varepsilon_1, \varepsilon_2, \dots\}$ be a decreasing sequence that satisfies $\varepsilon_t \leq \min_a \pi_t(a)$ (we say that $\pi_t(a)$ is \emph{bounded from below} by $\varepsilon_t$). In the appendix we prove the following upper bound on $V_t(a)$.
\begin{lemma}
\label{lem:V}
For all $t$ and $a$$:$
\[
V_t(a) \leq \frac{2t}{\varepsilon_t}.
\]
\end{lemma}

For a distribution $\rho$ over ${\cal A}$ define the expected and empirical regret of $\rho$ as $\Delta(\rho) = \mathbb E_{\rho(a)} [\Delta(a)]$ and $\hat \Delta_t(\rho) = \mathbb E_{\rho(a)} [\hat \Delta_t(a)]$. The following theorem follows immediately from Theorem \ref{thm:PAC-Bayes-Bernstein} and Lemma \ref{lem:V} by taking a uniform prior over the actions.
\begin{theorem}
\label{thm:PAC-Bayes-Delta}
For any sequence of sampling distributions $\{\pi_1,\pi_2,\dots\}$ that are bounded from below by a decreasing sequence $\{\varepsilon_1, \varepsilon_2, \dots\}$ that satisfies
\begin{equation}
\label{eq:technical-epsilon}
\frac{\ln(K) + 2 \ln(t+1) + \ln \frac{2}{\delta}}{2(e-2)t} \leq \varepsilon_t,
\end{equation}
where $\pi_t$ can depend on ${\cal T}_{t-1}$, for all possible distributions $\rho_t$ given $t$ and for all $t \geq 1$ simultaneously with probability greater than $1-\delta$$:$
\begin{equation}
\label{eq:PAC-Bayes-Delta-Bernstein}
\left |\Delta(\rho_t) - \hat \Delta_t(\rho_t) \right | \leq 2\sqrt{\frac{2(e-2)\left(\ln(K) + 2 \ln(t+1) + \ln \frac{2}{\delta} \right)}{t\varepsilon_t}}.
\end{equation}
\end{theorem}

\begin{proof}
For a uniform prior $\mu_t(a) = \frac{1}{K}$ we have $KL(\rho_t\|\mu_t) \leq \ln(K)$. By Lemma \ref{lem:V}, for any $\rho_t$ the weighted cumulative variance is bounded by $V_t(\rho_t) \leq \frac{2t}{\varepsilon_t}$. By taking $\lambda_t = \sqrt{\frac{\ln(K) + 2 \ln(t+1) + \ln \frac{2}{\delta}}{2(e-2)t}}$ and substituting the bounds on $KL(\rho_t\|\mu_t)$ and $V_t(\rho_t)$ into \eqref{eq:PAC-Bayes-Bernstein} we obtain \eqref{eq:PAC-Bayes-Delta-Bernstein}. (We considered the martingales $t(\Delta(a) - \hat \Delta_t(a))$, which provided a factor of $t$ in the denominator.) The technical condition \eqref{eq:technical-epsilon} follows from the requirement \eqref{eq:technical} on $\lambda_t$.
\end{proof}

{\bf Remarks:} Theorem \ref{thm:PAC-Bayes-Delta} provides an improvement over the corresponding Theorems 2 and 3 in the precursory report \cite{SLST+11} by decreasing the dependence on $\varepsilon_t$ from $1/\varepsilon_t$ to $1/\sqrt{\varepsilon_t}$. This in turn makes it possible to improve the regret bound, which is shown next. Interestingly, the uniform prior $\mu_t$ yields a tighter (and also simpler) bound than a distribution-dependent prior used in \citet{SLST+11}. It also broadens the range of playing strategies for which the regret bound given in Theorem \ref{thm:regret-delta} holds. We note that the uniform prior neutralizes the power of PAC-Bayesian analysis to discriminate between different hypotheses. For problems with richer structure studied in the follow-up paper \cite{SAL+11}, more interesting priors can be defined that yield advantages over alternative approaches. The multiarmed bandit problem studied here is, nevertheless, important for the development of the new tool.

We note that in the next theorem we take $\varepsilon_t = K^{-2/3}t^{-1/3}$ and the technical condition \eqref{eq:technical-epsilon} is satisfied for $t$ that is slightly larger than $K (\ln(K) + \ln \frac{2}{\delta})^{3/2}$.

\begin{theorem}
\label{thm:regret-delta}
Let $\varepsilon_t = K^{-2/3}t^{-1/3}$ and take any $\gamma_t$, such that $\gamma_t \geq K^{-1/3}t^{1/3}\sqrt{\ln K}$. For $t < K$ let $\pi_t(a) = \frac{1}{K}$ for all $a$ and for $t \geq K$ let
\[
\pi_{t+1}(a) = \tilde \rho_t^{_{exp}}(a) = (1 - K \varepsilon_{t+1}) \rho_t^{_{exp}}(a) + \varepsilon_{t+1},
\]
where
\[
\rho_t^{_{exp}}(a) = \frac{1}{Z(\rho_t^{_{exp}})} e^{\gamma_t \hat R_t(a)}
\]
and
\[
Z(\rho_t^{_{exp}}) = \sum_a e^{\gamma_t \hat R_t(a)}.
\]
Then the expected per-round regret $\Delta(\tilde \rho_t^{_{exp}}) = R(a^*) - R(\tilde \rho_t^{_{exp}})$ is bounded by$:$
\[
\Delta(\tilde \rho_t^{_{exp}}) \leq \frac{K^{1/3}}{(t+1)^{1/3}}
\left (1 + \sqrt{\ln K}+ 2\sqrt{2(e-2)\left (\ln(K) + 2 \ln(t+1) + \ln \frac{2}{\delta} \right)} \right )
\]
with probability greater than $1-\delta$ simultaneously for all rounds $t$, where $t$ satisfies \eqref{eq:technical-epsilon} $($which means that $t \geq K \left (\frac{\ln(K) + 2 \ln(t+1) + \ln \frac{2}{\delta}}{2(e-2)} \right )^{3/2}$, note that $t$ also appears on the right hand side$)$. This translates into a total regret of $\tilde O(K^{1/3}t^{2/3})$ $($where $\tilde O$ hides logarithmic factors$)$.
\end{theorem}

For $\gamma_t = \varepsilon_t^{-1}$ the playing strategy in Theorem \ref{thm:regret-delta} is known as the EXP3 algorithm for adversarial bandits \cite{ACB+02}, which is applied here to stochastic bandits. When $\gamma_t$ tends to infinity, we obtain the $\varepsilon$-greedy algorithm for stochastic bandits \cite{ACBF02}. Theorem \ref{thm:regret-delta} covers the spectrum of all possible intermediate strategies.

\section{Towards a Tighter Regret Bound}
\label{sec:exp}

We note that there is still a room for improvement, which we believe will enable to achieve regret bounds of order $\tilde O(\sqrt{Kt})$. The main source of looseness is the usage of the crude global upper bound $\frac{2t}{\varepsilon_t}$ on the cumulative variances in Lemma \ref{lem:V} that holds for any distribution $\rho_t$. While this bound seems to be tight for the $\varepsilon$-greedy strategy, we believe that it can be tightened for the EXP3 algorithm. It is possible to show that if we play according to the distributions $\{\tilde \rho_1^{_{exp}}, \dots, \tilde \rho_t^{_{exp}}\}$, then for ``good'' actions $a$ (those for which $\Delta(a) \leq \frac{1}{\gamma_t}$) the cumulative variance $V_t(a)$ is bounded by $C K t$ for some constant $C$. If we could show that for ``bad'' actions $a$ (those for which $\Delta(a) > \frac{1}{\gamma_t}$) the probability $\rho_t^{_{exp}}$ of picking such actions is bounded by $C \varepsilon_t$, then the cumulative variance $V_t(\rho_t^{_{exp}})$ would be bounded by $C K t$. This is, in fact, true for ``very bad'' actions (those, for which $\Delta(a)$ is close to 1), but it does not hold for actions with $\Delta(a)$ close to $\frac{1}{\gamma_t}$. However, we can possibly show that for such actions $\rho_t^{_{exp}}(a) \leq C \varepsilon_t$ for most of the rounds ($1 - \varepsilon_t$ fraction will suffice) and then we will be able to achieve $\tilde O(\sqrt{Kt})$ regret. In the experiment that follows we provide an empirical evidence that this conjecture holds in practice.

Another possible approach is to apply the EXP3.P algorithm of \citet{ACB+02}. However, in the experiment that follows we show that in the stochastic setting EXP3 algorithm achieves much lower regret than EXP3.P. It is, therefore, worth exploring the first route. We also note that \citet{ACB+02} do not provide an explicit bound on the variance of EXP3.P, which is required for our bound. This would have to be done for the second way of achieving $\tilde O(\sqrt{Kt})$ regret bound.

\subsection{Empirical Test Study}

In the following experiment we show that in the stochastic setting EXP3 algorithm achieves lower regret compared to EXP3.P.1 algorithm of \citet{ACBF02}. We also show that the variance of EXP3 algorithm is reasonably close to $2Kt$. Finally, we show that in the stochastic setting the regret of EXP3 algorithm is comparable or even lower than the regret of UCB strategy \cite{ACBF02} in the short run, but gets worse in the long run. We note that UCB strategy is not compatible with PAC-Bayesian analysis, since in UCB every action has its own sample size and the sample size of ``bad'' actions grows sublinearly with the number of game rounds. Designing a strategy that would be compatible with PAC-Bayesian analysis and achieve the regret of UCB in the long run is an important direction for future research.

\subsubsection*{Experiment Setup}

We took a 2-arm bandit problem with biases 0.5 and 0.6 for the two arms and ran EXP3 algorithm from Theorem \ref{thm:regret-delta} with $\varepsilon_t = 1 / \sqrt{Kt}$ and $\gamma_t = \sqrt{t \ln K / K}$, EXP3.P.1 algorithm of \citet{ACB+02} with $\delta = 0.001$, and UCB1 algorithm of \citet{ACBF02}. In the first experiment we made 1000 repetitions of the game and in each game we ran each of the algorithms for 10,000 rounds. In the second experiment we made 100 repetitions of the game and in each game we ran each of the algorithms for $10^7$ rounds. In Figure \ref{fig:exp} we show:
\begin{itemize}
  \item[\ref{fig:exp}.a] Experiment 1 ($10^4$ rounds): Average (over 1000 repetitions of the game) cumulative regret of EXP3, EXP3.P.1, and UCB1 algorithms. 
	\item[\ref{fig:exp}.b] Experiment 1: Average cumulative variance of EXP3 and EXP3.P.1 normalized by $2 K t$, which is what we would like it to be: $\frac{1}{2 K t} \cdot \frac{1}{1000} \sum_{i=1}^{1000} V_t^i(\rho_t)$, where $i \in [1,\dots,1000]$ indexes the experiments.
  \item[\ref{fig:exp}.c] Experiment 2 ($10^7$ rounds): Average (over 100 repetitions of the game) cumulative regret of EXP3 and UCB1 algorithms. The regret of EXP3.P.1 algorithm was far above the regret of EXP3 and UCB1 and, therefore, was omitted from the graphs.
	\item[\ref{fig:exp}.d] Experiment 2: Average cumulative variance of EXP3 normalized by $2 K t$.
\end{itemize}

\begin{figure}[t]
\subfigure[Cumulative Regret, $10^4$ rounds]{\includegraphics[width=.49\columnwidth]{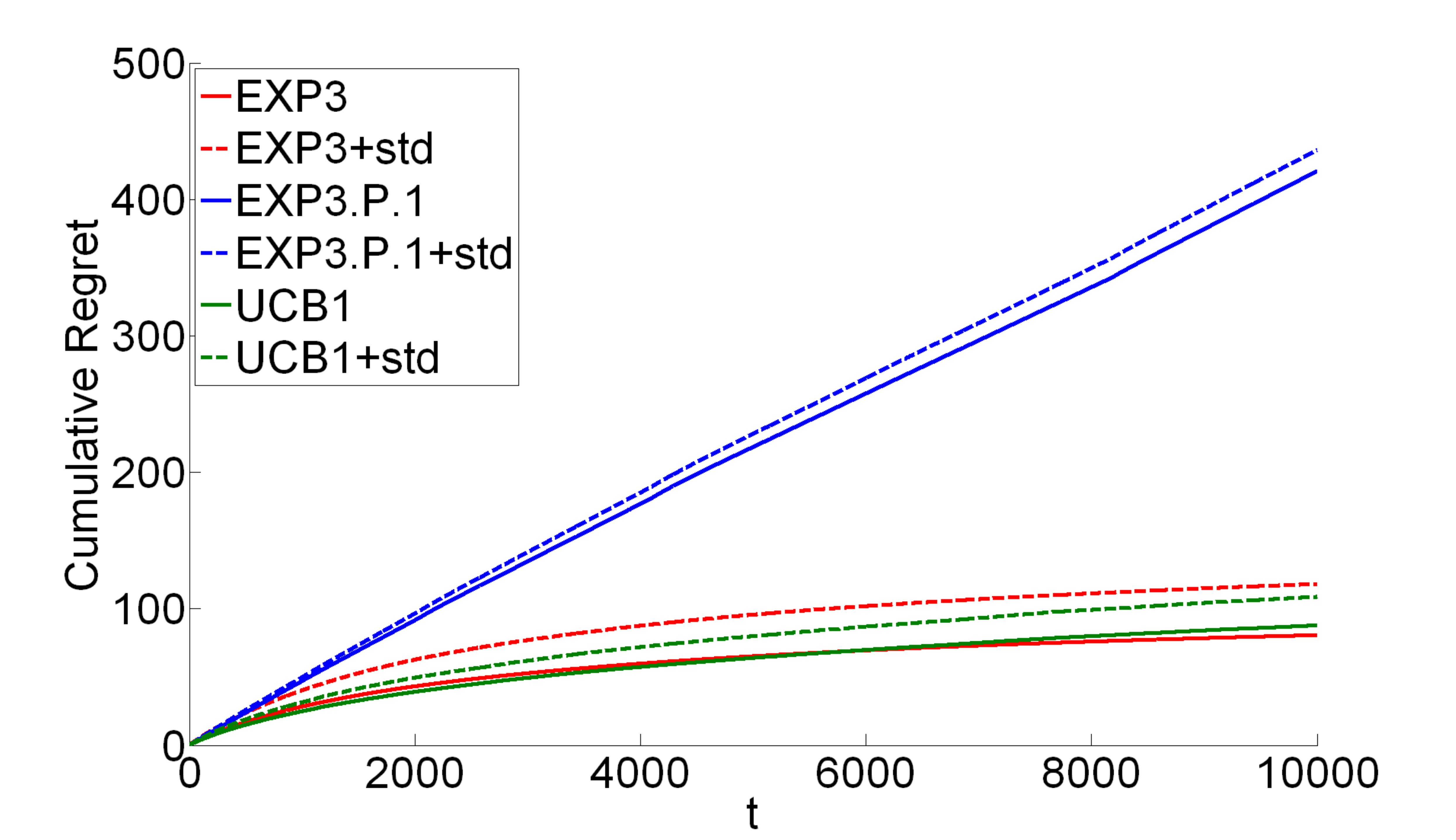}}
\quad
\subfigure[$\frac{1}{2Kt} ~\cdot$ (Cumulative Variance), $10^4$ rounds]{\includegraphics[width=.49\columnwidth]{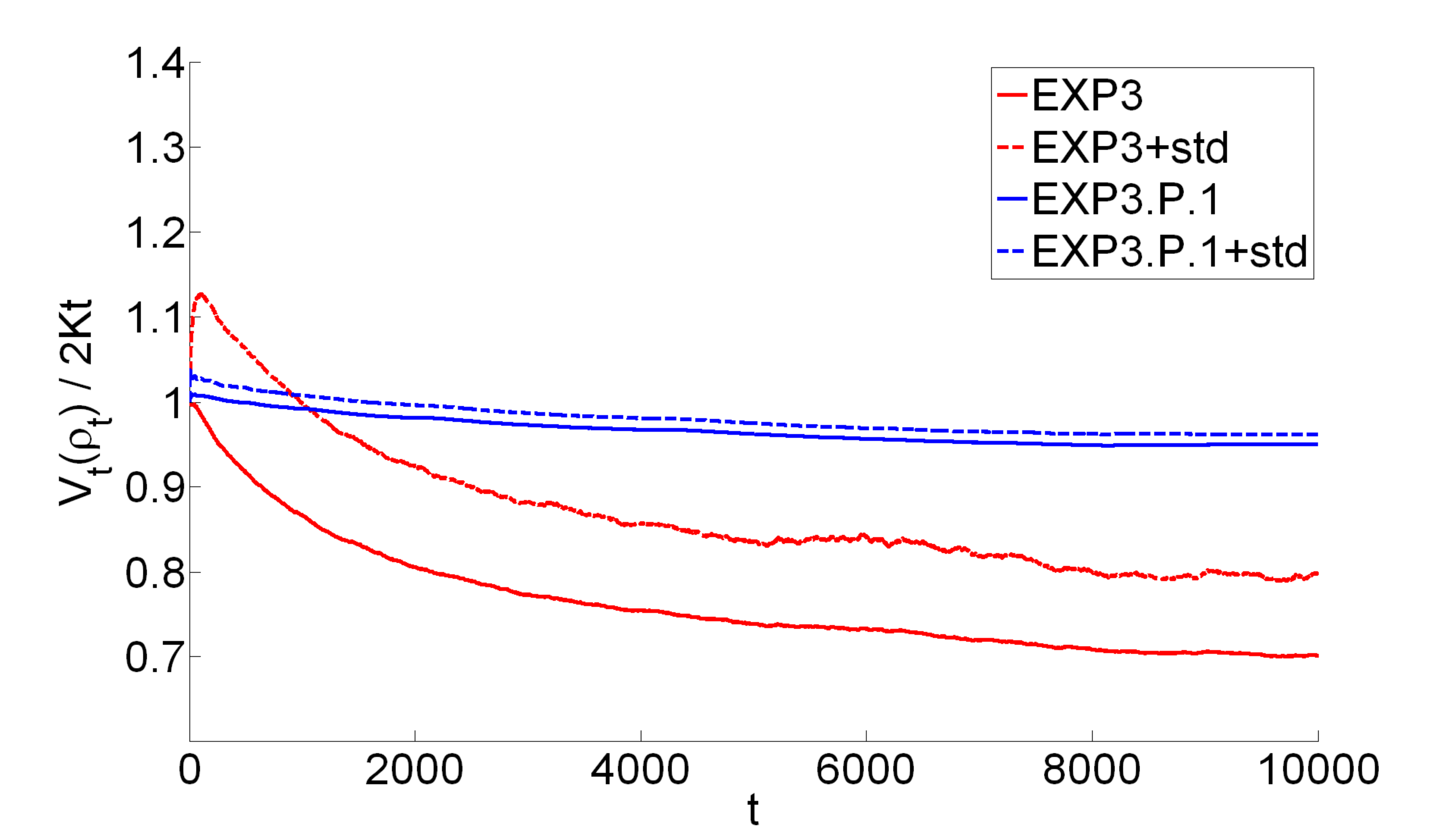}}\\%
\subfigure[Cumulative Regret, $10^7$ rounds]{\includegraphics[width=.49\columnwidth]{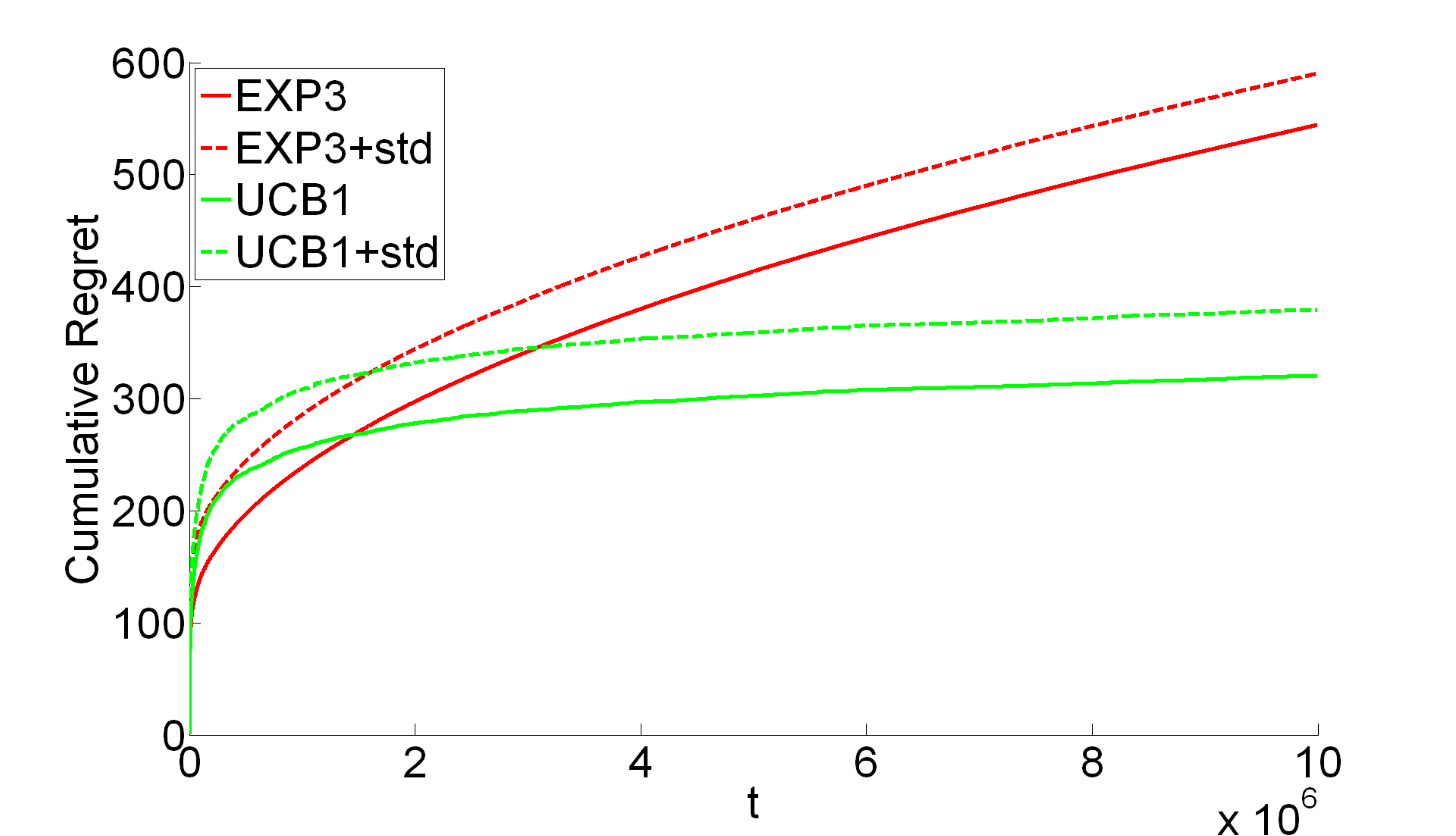}}%
\quad
\subfigure[$\frac{1}{2Kt} ~\cdot$ (Cumulative Variance), $10^7$ rounds]{\includegraphics[width=.49\columnwidth]{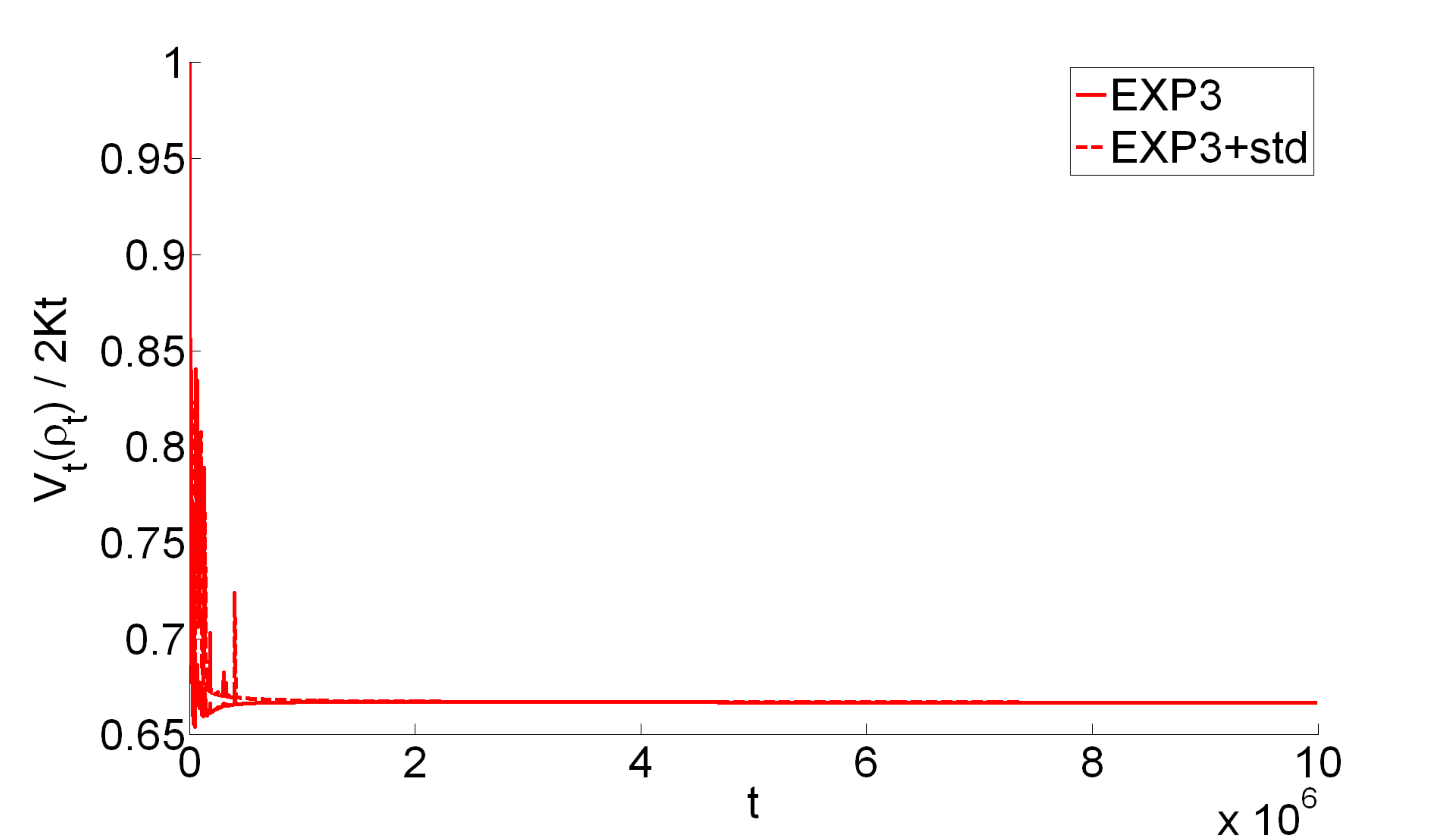}}%
\caption{{\bf Experimental results.} Solid lines show mean values over experiment repetitions, dotted lines show mean values plus one standard deviation (std).}%
\label{fig:exp}%
\end{figure}

\subsubsection*{Observations}

\begin{enumerate}

\item In the stochastic setting the performance of EXP3 is significantly superior to the performance of EXP3.P.1.

\item In the stochastic setting, the performance of EXP3 is comparable or even superior to the performance of UCB1 in the short run, but becomes worse than the performance of UCB1 in the long run (beyond $2 \cdot 10^6$ iterations). The reason is that the number of pulls of the suboptimal arm are roughly $\sqrt t$ for EXP3 and $\ln(t) / \Delta(a)^2$ for UCB. In our experiment $\Delta(a) = 0.1$ for the suboptimal arm, thus  $\sqrt{t} > \ln(t) / \Delta(a)^2$ when $t > \ln(t)^2 / \Delta(a)^4$, which holds when $t > 2 \cdot 10^6$.

\item In the stochastic setting, the variance of EXP3 is initially higher than the variance of EXP3.P.1, but eventually it becomes lower.

\item Initially the variance of EXP3 is just slightly above $2 K t$ (by a factor of less than 2) and eventually it stabilizes around $0.66 \cdot 2 K t$ for the problem that we considered.

\end{enumerate}

\section{Discussion}
\label{sec:dis}

We presented a new framework for data-dependent analysis of the exploration-exploita\-tion trade-off and for simultaneous analysis of model order selection and the exploration-exploitation trade-off. We note that model order selection does not come up in the multiarmed bandit problem due to simplicity of the structure of this problem. Nevertheless, the multiarmed bandit problem is a convenient playground for the development of the new tool. In the follow-up paper we show that the new technique developed here can be applied to multiarmed bandits with side information and yield an advantage over state-of-the-art \cite{SAL+11}.

An important direction for future research is to tighten Theorems \ref{thm:PAC-Bayes-Delta} and \ref{thm:regret-delta}, so that the regret bound will match state-of-the-art regret bounds obtained by alternative techniques. We believe that the ideas described in Section \ref{sec:exp} can make it possible. The experiments presented in Section \ref{sec:exp} show that empirically in the stochastic setting our algorithm is significantly superior to state-of-the-art algorithms for adversarial bandits and slightly worse than state-of-the-art algorithms for stochastic bandits. Closing the gap with state-of-the-art algorithms for stochastic bandits is another important direction for future research.

Other directions for future research include application of our framework to Markov decision processes \cite{FP10}, active learning \cite{BDL09}, and problems with continuous state and action spaces, such as Gaussian process bandits \cite{SKKS10}.

\appendix

\section{Proofs}

In this appendix we provide the proofs of Theorems \ref{thm:PAC-Bayes-Bernstein} and \ref{thm:regret-delta} and Lemma \ref{lem:V}.

\subsection{Proof of Theorem \ref{thm:PAC-Bayes-Bernstein}}

The proof of Theorem \ref{thm:PAC-Bayes-Bernstein} relies on the following two lemmas. The first one is a Bernstein-type inequality. For a proof of Lemma \ref{lem:Bernstein} see, for example, the proof of Theorem 1 in \citet{BLL+11}.
\begin{lemma}[Bernstein's inequality]
\label{lem:Bernstein}
Let $X_1,\dots,X_t$ be a martingale difference sequence $($meaning that $\mathbb E [X_\tau | X_1,\dots,X_{\tau-1}] = 0$ for all $\tau$$)$, such that $X_\tau \leq C$ for all $\tau$ with probability 1. Let $M_t = \sum_{\tau = 1}^t X_\tau$ be a corresponding martingale and $V_t = \sum_{\tau = 1}^t \mathbb E [X_\tau^2 | X_1,\dots,X_{\tau - 1}]$ be the cumulative variance of this martingale. Then for any fixed $\lambda \in [0,\frac{1}{C}]$$:$
\[
\mathbb E e^{\lambda M_t - (e - 2) \lambda^2 V_t} \leq 1.
\]
\end{lemma}

The second lemma originates in statistical physics and information theory \cite{DV75,DE97,Gra11} and forms the basis of PAC-Bayesian analysis. See \cite{Ban06} for a proof.
\begin{lemma}[Change of measure inequality]
\label{lem:PAC-Bayes}
For any measurable function $\phi(h)$ on ${\cal H}$ and any distributions $\mu(h)$ and $\rho(h)$ on ${\cal H}$, we have$:$
\[
\mathbb E_{\rho(h)}[\phi(h)] \leq KL(\rho\|\mu) + \ln \mathbb E_{\mu(h)}[ e^{\phi(h)}].
\]
\end{lemma}

Now we are ready to state the proof of Theorem \ref{thm:PAC-Bayes-Bernstein}.
\begin{proof} {\bf of Theorem \ref{thm:PAC-Bayes-Bernstein}} 
Take $\phi(h) = \lambda_t M_t(h) - (e-2) \lambda_t^2 V_t(h)$ and $\delta_t = \frac{1}{t(t+1)} \delta \geq \frac{1}{(t+1)^2}\delta$. (It is well-known that $\sum_{t=1}^\infty \frac{1}{t(t+1)} = \sum_{t=1}^\infty \left (\frac{1}{t} - \frac{1}{t+1}\right ) = 1$.) Then the following holds for all $\rho_t$ and $t$ simultaneously with probability greater than $1 - \frac{\delta}{2}$:
\begin{align}
\lambda_t M_t(\rho_t) - (e-&2)\lambda_t^2 V_t(\rho_t)=\mathbb E_{\rho_t(h)} [\lambda_t M_t(h) - (e-2) \lambda_t^2 V_t(h)]\label{eq:1}\\
&\leq KL(\rho_t\|\mu_t) + \ln \mathbb E_{\mu_t(h)} [e^{\lambda_t M_t(h) - (e-2) \lambda_t^2 V_t(h)}]\label{eq:2}\\
&\leq KL(\rho_t\|\mu_t) + 2 \ln(t+1) + \ln \frac{2}{\delta}
+ \ln \mathbb E_{{\cal T}_t} \mathbb E_{\mu_t(h)} [e^{\lambda_t M_t(h) - (e-2) \lambda_t^2 V_t(h)}]\label{eq:3}\\
&= KL(\rho_t\|\mu_t) + 2 \ln(t+1) + \ln \frac{2}{\delta}
+ \ln \mathbb E_{\mu_t(h)} \mathbb E_{{\cal T}_t} [e^{\lambda_t M_t(h) - (e-2) \lambda_t^2 V_t(h)}]\label{eq:4}\\
&\leq KL(\rho_t\|\mu_t) + 2 \ln(t+1) + \ln \frac{2}{\delta},\label{eq:5}
\end{align}
where \eqref{eq:1} is by definition of $M_t(\rho_t)$ and $V_t(\rho_t)$, \eqref{eq:2} is  by Lemma \ref{lem:PAC-Bayes}, \eqref{eq:3} holds with probability greater than $1-\frac{\delta}{2}$ by Markov's inequality and a union bound over $t$, \eqref{eq:4} is due to the fact that $\mu_t$ is independent of ${\cal T}_t$, and \eqref{eq:5} is by Lemma  \ref{lem:Bernstein}.

By applying the same argument to martingales $-M_t(h)$ and taking a union bound over the two we obtain that with probability greater than $1-\delta$:
\[
|M_t(\rho_t)| \leq \frac{KL(\rho_t\|\mu_t) + 2 \ln(t+1) + \ln \frac{2}{\delta}}{\lambda_t} + (e-2) \lambda_t V_t(\rho_t),
\]
which is the statement of the theorem. The technical condition \eqref{eq:technical} follows from the requirement that $\lambda_t \in [0,\frac{1}{C_t}]$.
\end{proof}

\subsection{Proof of Lemma \ref{lem:V}}

\begin{proof} {\bf of Lemma \ref{lem:V}}
\begin{align}
V_t(a) &= \sum_{\tau=1}^t \mathbb E[([R_\tau^{a^*} - R_\tau^a] - [R(a^*) - R(a)])^2 | {\cal T}_{\tau-1}]\notag\\
&= \left (\sum_{\tau=1}^t \mathbb E[(R_\tau^{a^*} - R_\tau^a)^2 | {\cal T}_{\tau-1}]\right ) - t\Delta(a)^2 \label{eq:61}\\
&\leq \left (\sum_{\tau=1}^t \left (\frac{\pi_\tau(a)}{\pi_\tau(a)^2} + \frac{\pi_\tau(a^*)}{\pi_\tau(a^*)^2} \right ) \right )\label{eq:6}\\
&= \left (\sum_{\tau=1}^t \left (\frac{1}{\pi_\tau(a)} + \frac{1}{\pi_\tau(a^*)} \right ) \right )\notag\\
&\leq \frac{2t}{\varepsilon_t},\label{eq:66}
\end{align}
where \eqref{eq:61} is due to the fact that $\mathbb E[R_\tau^a | {\cal T}_{\tau -1}] = R(a)$, \eqref{eq:6} is due to the fact that $R_t \leq 1$ and $t\Delta(a)^2 \geq 0$, and \eqref{eq:66} is due to the fact that $\frac{1}{\pi_\tau(a)} \leq \frac{1}{\varepsilon_t}$ for all $a$ and $1 \leq \tau \leq t$.
\end{proof}

\subsection{Proof of Theorem \ref{thm:regret-delta}}

\begin{proof} {\bf of Theorem \ref{thm:regret-delta}}
We use the following regret decomposition:
\begin{equation}
\Delta(\tilde \rho_t^{_{exp}}) = [\Delta(\rho_t^{_{exp}}) - \hat \Delta_t(\rho_t^{_{exp}})] + \hat \Delta_t(\rho_t^{_{exp}}) + [R(\rho_t^{_{exp}}) - R(\tilde \rho_t^{_{exp}})].\label{eq:regret-delta}
\end{equation}

The first term in the decomposition is bounded by Theorem \ref{thm:PAC-Bayes-Delta}. Before bounding the middle term in \eqref{eq:regret-delta} we bound the last term, which is much simpler, and then return to the middle term. The bound on $[R(\rho_t^{_{exp}}) - R(\tilde \rho_t^{_{exp}})]$ is achieved by the following lemma.
\begin{lemma}
\label{lem:eps}
Let $\tilde \rho$ be an $\varepsilon$-smoothed version of $\rho$, such that
\[
\tilde \rho(a) = (1 - K \varepsilon) \rho(a) + \varepsilon.
\]
Then
\begin{equation}
R(\rho) - R(\tilde \rho) \leq K \varepsilon.
\label{eq:tilde}
\end{equation}
\end{lemma}

\begin{proof}
\begin{align}
R(\rho) - R(\tilde \rho) &= \sum_a (\rho(a) - \tilde \rho(a)) R(a)\notag\\
&\leq \frac{1}{2} \sum_a |\rho(a) - \tilde \rho(a)|\label{eq:Rl1}\\
&= \frac{1}{2} \sum_a |\rho(a) - (1 - K \varepsilon) \rho(a) - \varepsilon|\notag\\
&= \frac{1}{2} \sum_a |K \varepsilon \rho(a) - \varepsilon|\notag\\
&\leq \frac{1}{2} K \varepsilon \sum_a \rho(a) + \frac{1}{2} K \varepsilon\notag\\
&= K \varepsilon.\notag
\end{align}
In \eqref{eq:Rl1} we used the fact that $0 \leq R(a) \leq 1$ and $\rho$ and $\tilde \rho$ are probability distributions.
\end{proof}

In the next lemma we bound $\hat \Delta(\rho_t^{_{exp}})$.
\begin{lemma}
\label{lem:delta}
\begin{equation}
\hat \Delta(\rho_t^{_{exp}}) \leq \frac{\ln K}{\gamma_t}.
\label{eq:delta-hat}
\end{equation}
\end{lemma}

\begin{proof}
Observe that by multiplying nominator and denominator in the definition of $\rho_t^{_{exp}}$ by $e^{-\gamma_t \hat R_t(a^*)}$ we obtain:
\[
\rho_t^{_{exp}}(a) = \frac{e^{\gamma_t \hat R_t(a)}}{Z(\rho_t^{_{exp}})} = \frac{e^{-\gamma_t \hat \Delta_t(a)}}{Z'(\rho_t^{_{exp}})},
\]
where $Z'(\rho_t^{_{exp}}) = \sum_a e^{-\gamma_t \hat \Delta_t(a)}$. The empirical regret $\hat \Delta_t(\rho_t^{_{exp}})$ then obtains the form:
\[
\hat \Delta_t(\rho_t^{_{exp}}) = \sum_a \rho_t(a) \hat \Delta_t(a) = \frac{\sum_a \hat \Delta_t(a) e^{-\gamma_t \hat \Delta_t(a)}}{\sum_a e^{-\gamma_t \hat \Delta_t(a)}}.
\]
The lemma follows from Lemma \ref{lem:expsum} below and the observation that $\hat \Delta_t(a^*) = 0$.
\end{proof}

\begin{lemma}
\label{lem:expsum}
Let $x_1 = 0$ and $x_2,\dots,x_n$ be $n-1$ arbitrary numbers. For any $\alpha > 0$ and $n \geq 2$$:$
\begin{equation}
\frac{\sum_{i=1}^n x_i e^{-\alpha x_i}}{\sum_{j=1}^n e^{-\alpha x_j}} \leq \frac{\ln(n)}{\alpha}.
\label{eq:expsum}
\end{equation}
\end{lemma}

\begin{proof}
Since negative $x_i$-s only decrease the left hand side of \eqref{eq:expsum} we can assume without loss of generality that all $x_i$-s are positive. Due to symmetry, the maximum is achieved when all $x_i$-s (except $x_1$) are equal:
\begin{equation}
\label{eq:x}
\frac{\sum_{i=1}^n x_i e^{-\alpha x_i}}{\sum_{j=1}^n e^{-\alpha x_j}} \leq \max_x \frac{(n-1) x e^{-\alpha x}}{1 + (n-1) e^{-\alpha x}}.
\end{equation}

We apply change of variables $y = e^{-\alpha x}$, which means that $x = \frac{1}{\alpha}\ln \frac{1}{y}$. By substituting this into the right hand side of \eqref{eq:x} we get
\[
\frac{(n-1) x e^{-\alpha x}}{1 + (n-1) e^{-\alpha x}} = \frac{1}{\alpha} \cdot \frac{(n-1) y \ln \frac{1}{y}}{1 + (n-1)y}.
\]
In order to prove the bound we have to show that $\frac{(n-1) y \ln \frac{1}{y}}{1 + (n-1)y} \leq \ln n$. 

By taking Taylor's expansion of $\ln z$ around $z = n$ we have:
\[
\ln z \leq \ln n + \frac{1}{n} (z - n) = \ln n + \frac{z}{n} - 1.
\]
Thus:
\begin{align}
\frac{(n-1) y \ln \frac{1}{y}}{1 + (n-1)y} &\leq \frac{(n-1) y (\ln n + \frac{1}{n y} - 1)}{1 + (n-1) y}\notag\\
&\leq \frac{y (n-1) \ln n + \frac{n-1}{n}}{(n-1) y + 1}\notag\\
&\leq \frac{(y (n-1) + 1) \ln n}{y(n-1) + 1}\label{eq:ln}\\
&= \ln n,\notag
\end{align}
where \eqref{eq:ln} follows from the fact that $\ln z \leq z - 1$ for any positive $z$, and hence $\ln \frac{1}{n} \leq \frac{1}{n} - 1$, which means that $\ln n \geq 1 - \frac{1}{n} = \frac{n-1}{n}$ for all $n > 0$.
\end{proof}

Substitution of \eqref{eq:PAC-Bayes-Delta-Bernstein}, \eqref{eq:tilde}, \eqref{eq:delta-hat}, and the choice of $\varepsilon_t$ and $\gamma_t$ in theorem formulation into \eqref{eq:regret-delta} concludes the proof.

\end{proof}

\section*{Acknowledgments}

This work was supported in part by the IST Programme of the European Community, under the PASCAL2 Network of Excellence, IST-2007-216886, and by the European Community's Seventh Framework Programme (FP7/2007-2013), under grant agreement $N^o$231495. This publication only reflects the authors' views.

{\small
\bibliography{bibliography}
}

\end{document}